%% file: tex/main.tex
\acmConference[RecSys]{Recommender Systems}{27th - 31st August 2017}{Como, Italy}

\title{Human Interaction with Recommendation Systems}

\author{Sven Schmit}
\affiliation{%
    \institution{Stanford University}
}
\email{schmit@stanford.edu}
\author{Carlos Riquelme}
\affiliation{%
    \institution{Stanford University}
}
\email{rikel@stanford.edu}

\thanks{This work is supported by the National Science Foundation.}

\date{\today}

\begin{document}

    \input{tex/abstract}

    \maketitle

\input{tex/intro}

    \input{tex/model}
    \input{tex/bias}

    \input{tex/exploration}

    \input{tex/simulations}

    \input{tex/conclusion}

    % removed for reviews
    % \input{tex/acknowledgements}

    \input{tex/supplement}

    \bibliographystyle{ACM-Reference-Format}
    \bibliography{bib/bibliography}

\end{document}

%% file: tex/abstract.tex
\begin{abstract}
    Many recommendation algorithms rely on user data to generate recommendations.
    However, these recommendations also affect the data obtained from future users.
    This work aims to understand the effects of this dynamic interaction.
    We propose a simple model where users with heterogeneous preferences
    arrive over time.
    Based on this model, we prove that naive estimators,
    i.e.\ those which ignore this feedback loop,
    are not consistent.
    We show that consistent estimators are efficient in
    the presence of myopic agents.
    Our results are validated using extensive simulations.
\end{abstract}

% Shortened abstract in registration
% Recommendation systems rely on historical user data to provide suggestions. We propose an explicit and simple model for the interaction between users and recommendations provided by a platform, and relate this model to the multi-armed bandit literature. First, we show that this interaction leads to a bias in naive estimators due to selection effects. This bias leads to suboptimal outcomes, which we quantify in terms of linear regret. We end the first part by discussing ways to obtain unbiased estimates. The second part of this work considers exploration of alternatives. We show that although agents are myopic, agents' heterogeneous preferences ensure that recommendation systems 'learn' about all alternatives without explicitly incentivizing this exploration. This work provides new and practical insights relevant to a wide range of systems designed to help users make better decisions.

%% file: tex/intro.tex
\section{Introduction}
\label{sec:intro}

We find ourselves surrounded by recommendations that help us make better decisions.
However, relatively little work has been devoted to the understanding of the
dynamics of such systems caused by the interaction with users.
This work aims to understand the dynamics that arise when users combine the recommendations with
their own preference when making a decision.

For example, a user of Netflix uses their recommendations to decide what movie to watch.
However, this user also has her own beliefs about movies, e.g.~based on artwork, synopsis, actors,
recommendations by friends, etc.
The user thus combines the suggestions from Netflix with her own preferences to decide what movie to watch.
Netflix captures data on the outcome to improve its recommendations to future users.
Of course, this pattern is not unique to Netflix, but observed more broadly;
across all platforms that use recommendations.

% We are interested in gaining a fundamental understanding of the process of
% combining suggestions from a recommendation system with preferences of individual users.
% Doing so, we hope to illuminate important aspects of such systems that
% cannot be exposed by the data that these systems collect.
% Every outcome of a user interacting with an item observed by a recommendation system is,
% almost tautologically, observed \emph{after} a user has selected the item.
% However, this data is used to provide recommendations to users \emph{before} they have made a selection.
% This dynamic feedback model is often overlooked.
% Fundamentally, this work investigates the effects of user interaction
% on two issues: \emph{consistency} and \emph{efficiency}.

A first requirement for any estimator is consistency;
however it is not clear that in the presence of human interaction
naive estimators are consistent.
Indeed, we show that simple estimators can easily be fooled by the selection effect of the users.
We propose to measure performance by adapting the notion of regret
from the multi-armed bandit literature.
Using this metric, we show that naive estimators suffer linear regret;
even with `infinite data' the performance per time step is bounded away from the optimum.

Using the notion of regret is useful as it also allows us to quantify the efficiency of estimators.
New users and items with little to no data constantly arrive and thus
a recommendation system is always in a state of learning.
It is therefore important that the system learn efficiently from data.
While this might sound like the well-known \emph{cold-start} problem,
that is not the focus of this work;
Rather than providing recommendation solutions for users in the absence of data,
we focus on quantifying how quickly an algorithm obtains enough data to make
good recommendations.
This is more akin to the social learning and incentivizing exploration literature
than work on the cold-start problem.

\subsection{Main results}
\label{sec:intro_main}

From a technical standpoint,
this paper provides a dynamical model that captures the dynamics
of users with heterogeneous preferences,
while abstracting away the specifics of recommendation algorithms.
In the first part of this work,
we show that there is a severe selection bias problem that leads to linear regret.
Second, we show that when the algorithm uses unbiased estimates for items,
`free' exploration occurs and we recover the familiar logarithmic regret bound.
This is important because inducing agents to explore is difficult from
both a statistical and strategic point of view.
We validate our claims using simulations with feature-based and low-rank methods.

It is important to note that the focus of this work is
to provide a simplified framework that allows us to reason
about the dynamic aspects of recommendation systems.
We do not claim that the model nor the assumptions
are a perfect reflection of reality.
Instead, we believe that the model we propose provides an excellent lens to
better understand vital aspects of recommendation systems.

\input{tex/related}

\subsection{Organization}
\label{sec:organization}

In the next section, we introduce our model.  In Sections \ref{sec:bias} and
\ref{sec:exploration} we focus on the issues of consistency and efficiency,
respectively.  We illustrate our results with simulations in
Section~\ref{sec:simulations} before concluding.

%% file: tex/related.tex
\subsection{Related work}
\label{sec:related}

This work roughly intersects with three separate fields of study.
\emph{Recommendation systems} \citep{Adomavicius2005TowardTN} have attracted
much attention.  In particular, much research has focused on new methods that
treat the data as fixed, rather than dynamic.  There has been less work on
selection bias,  which was first demonstrated by \citet{Marlin2003NIPS}, and
subsequent work \citep{amatriain2009like, Marlin2007UAI, Steck2010TrainingAT}.
Rather than modeling user behavior directly, they impose the statistical
assumption of a covariance shift; the distribution of observed ratings is not
altered by conditioning on the selection event, but five star ratings are more
likely to be observed.  More recently, \citet{Schnabel2016} and
\citet{Joachims2017UnbiasedLW} link the bias from covariance shifts to recent
advances in causal inference.  \citet{Mackey2010MixedMM} combine matrix
factorization and latent factor models to capture heterogeneity in interactions
and context.

The different approach of this work is reminiscent of the work on \emph{social
learning} \citep{chamley2004, smith2000pathological}, where agents learn about
the state of the world by combining their private signals with observations of
actions (but not necessarily outcomes) of others.
% The seminal work of \citet{smith2000pathological} shows that agents do not
% necessarily converge to the optimum outcome.
The work of \citet{Ifrach2014} is closest related to our setup.  They discuss
how consumer reviews converge on the quality of a product, given diversity of
preferences under a reasonable price assumption.
% Add Asu citation
However, the work in social learning focuses on users interacting with a single
item.  This seemingly minor minor difference leads to completely different
dynamics.

Finally, we can relate our work on exploration to the multi-armed bandit
literature \citep{Bubeck2012CoRR}.  In particular, there has been prior work on
human interaction with multi-armed bandit algorithms: for example, how a system
can optimally induce myopic agents to explore \citep{Kremer2013ImplementingT}
by using payments \citep{Frazier2014SIGECOM} or by the way the system
disseminates information \citep{papanastasiou2014crowdsourcing, Slivkins15,
MansourEC16}.  Similar to those works, we use the regret framework to analyse a
system with interacting agents.  Because in our model agents have heterogeneous
preferences, we show that agents do not need to be incentivized to explore.
Recent work by \citet{bastani2017exploiting, Qiang2016DynamicPW} consider
natural exploration in contextual bandit problems, and show that a modified
greedy algorithm performs well.  While their motivation is different, the
results are similar to ours.  There has also been work on `free exploration' in
auction environments \citep{Hummel2014WWW}.

%% file: tex/model.tex
\section{Modeling human-algorithm interaction}
\label{sec:model}

In this section, we propose a model for the interaction between
the recommendation system (\emph{platform}) and users (\emph{agents}).
Each user selects one of the items the platform recommends,
and reports their experience to the platform by providing a rating as feedback.
The platform uses this feedback to update the recommendations for the next user.

More formally, we assume there are $K$ items, labeled $i = 1, \ldots, K$,
and each item has a distinct, but unknown, quality $Q_i \in \reals$.
This aspect models the vertical differentiation between items
and it is the task of the platform to estimate these qualities.
For notational convenience, we assume $Q_1 > Q_2 > \ldots > Q_K$.
At every time step $t = 1, \ldots, T$ a new user arrives
and selects one of the $K$ items.
To do so, the user receives a private preference signal $\theta_{it} \sim F_i$ for each item,
drawn from a preference distribution $F_i$ which we make precise later.
The value of item $i$ for user $t$ is
\begin{equation}
	V_{it} = Q_i + \theta_{it} + \epsilon_{it}
	\label{eqn:value}
\end{equation}
where $\epsilon_{it}$ is additional noise drawn independently from a noise distribution $E$
with mean $0$ and finite variance $\sigma_E^2 < \infty$.
To aid the agents, the platform provides a recommendation score $s_{it}$,
aggregating the feedback from previous agents.
The agent uses her own preferences, along with the score, to select item $a_t$ according to
\[
	a_t = \arg\max_i s_{it} + \theta_{it}.
	\label{eqn:selection_rule}
\]
Hence, we make the assumption that the agent is boundedly rational
and uses $s_{it}$ as a surrogate for the quality.
Abusing notation, we write
\[
	V_t = V_{a_t t} = Q_{a_t} + \theta_{a_t t} + \epsilon_{a_t t}
\]
for the value of the chosen item for agent $t$.
After the agent selects item $a_t$ and observes the value $V_t$, the
platform queries for feedback $W_t$ from the user.
For example the platform can ask the user to provide the value of the item as a rating,
in which case $W_t = V_t$.
Note that the private preferences $\theta_t$ of the agent remain hidden.
The platform uses this feedback to give recommendations to future users.
In particular, we require $s_{it}$ to be measurable with respect to past feedback,
that is $\sigma\{a_\tau, W_\tau : \tau < t\}$.

We measure the performance of
a recommendation system in terms of (pseudo-)regret:
\[
	R_T = \sum_{t=1}^T \max_i (Q_i + \theta_{it}) - (Q_{a_t} + \theta_{a_t t}),
\]
which sums the difference between the expected value of the best item
and the expected value of the selected item.\footnote{
	Unlike the traditional bandit setting, there is no single best item.
	Rather, different users might have different optimal items.
}
We note that if scores $s_{it} \equiv Q_i$ for all $t$, the regret of such platform would be $0$,
as each user selects her optimal action using equation (\ref{eqn:selection_rule}).
% We stress that this notion of regret serves as a tool,
% rather than a goal; we are not particularly interested in designing an algorithm that
% minimizes this notion of regret, but rather use regret to analyse simple algorithms,
% and gain insights into the dynamics of recommendation systems in general.

\subsection{Preferences, values and personalization}

We use this section to expand on the motivation of the proposed model.
The value for item $i$ at time $t$ consists of three parts (see equation (\ref{eqn:value})).
First, the intrinsic quality $Q_i$ can be seen as the mean quality across users.
In our theoretical analysis we treat this as a constant to be estimated such that
we are able to disentangle the model fitting from the dynamics of interaction.
In this simple setting, one should view it as a vertical differentiator between items.
Taking hotels as example, it could model quality of service and cleanliness,
where a common ranking across agents is sensible.
The intrinsic qualities $Q_i$ can be replaced by more complicated models,
for example based on feature based regression methods, or matrix factorization methods.
Indeed, in Section~\ref{sec:mf} we provide simulation results where we replace
$Q_i$ with a low rank matrix factorization model.

The second term in the value equation, $\theta_{it}$, models
horizontal differentiation across agents.
In our simplified model agents only arrive once, and thus
this also covers different contexts.
For example, one traveler prefers a hotel on the waterfront, while
another prefers a hotel downtown, and yet a third prefers staying close to
the convention center.
While these hypothetical hotels could have the same quality,
the value for users differs, in ways known to the user.
However, the intrinsic quality of these properties are unknown to these users.

All in all, the value of item $i$ for agent $t$ is drawn from a distribution
with mean $Q_i$, and where the variance consists of a part
that is known to the user ($\theta_{it}$) and a part that
is unknown to both platform and user ($\epsilon_{it}$).
In section~\ref{sec:simulations} we investigate how well the theoretical
results carry over to more general models.

One could argue that personalization methods
(i.e. replacing $Q$ with more sophisticated models) supersede the need
for idiosyncratic preferences $\theta_{it}$, as these preferences
can be captured by those models.
However, we argue that in most cases this factor cannot be eliminated.
% There are two reasons why some part of the user's preferences cannot be captured by covariates and historical data.
Every recommendation system is constrained in terms of
the \emph{quantity} and \emph{quality} of the data it is based on.
First, a user only interacts with a system so often,
and that limits the amount of personalization that models can achieve.
% Even if there are models that could capture all of the user's preferences,
% in practice systems lack the data to support such models.
% This implies that some part of the preferences remain beyond the scope of models.
Second, recommendation systems often have access to only weak features,
and some aspects of user preferences and contexts,
such as taste or style, can be difficult to capture.
Together, these constraints make it difficult to fully model users preferences,
hence the need to explicitly model the unobserved preferences
to get a deeper understanding of the dynamics of recommendations systems.

\subsection{Incentives}

We note that the agents in our model are boundedly rational:
their behavior is not optimal, and in particular ignores the design of the platform.
Experimentally, there has been abundant evidence of human behavior that
is not rational \citep{camerer1998bounded, kahneman2003maps}.
Simple heuristics of user behavior have been used by others in the social
learning community.
Examples include learning about technologies from word-of-mouth interactions
\citep{ellison1993rules, ellison1995word}
and modeling persuasion in social networks \citep{Demarzo2003PersuasionBS}.
The combination of machine learning and mechanism design with boundedly rational agents
is explored by \cite{liu2015mechanism}.

The behavior of the user in our model implicitly relies on three assumptions:
\begin{enumerate} \itemsep0em
	\item The user is naive;
		she beliefs the scores supplied by the platform are
		unbiased estimates of the true quality.
	\item The user is myopic;
		she selects the item that seems best for her.
	\item The user has incentives to give honest feedback.
\end{enumerate}
The first assumption seems unrealistic
if the platform abuses this power to dictate exploration, which does not align
with the myopic behavior.
However, in Section \ref{sec:exploration} we show that there is no need for such
aggresive exploration from the platform to obtain order-optimal performance.
We also note that if the platform outputs the true qualities $Q_i$,
then the selection rule (\ref{eqn:selection_rule}) is optimal for a myopic agent.
Finally, it is not obvious why a myopic user would leave feedback.
While we do not explicitly model returning users,
we argue that in general a user is motivated to leave feedback because
it leads to better recommendations for her in the future.

% We do admit that these assumptions cut some corners for the sake of simplicity,
% but they do seem like a good proxy for behavior in practice.

% mention truncated lists?

%% file: tex/bias.tex
\section{Consistency}
\label{sec:bias}

In this section, we analyse the performance of standard algorithms, that is,
scoring processes that do not take into account that agents have private
preferences, and base the scores on empirical averages.  This does include
algorithms that trade-off exploration and exploitation, such as variants of UCB
\citep{Auer2002MachineLearning} and Thompson Sampling \citep{Russo2016AnIA}.
We focus on the Bernoulli preferences model, though in
Section~\ref{sec:simulations} we empirically demonstrate that different
preference distributions lead to similar outcomes.

First we define the set of agents before time $t$ that have selected item $i$ by
\[
    S_{it} = \{ \tau < t: a_\tau = i \}.
\]
We also define $\bar{V}_{it}$ to denote the empirical average of item $i$ up to time $t$:
\[
    \bar{V}_{it} = \frac{1}{|S_{it}|} \sum_{\tau \in S_{it}} V_\tau
\]
where we use $\bar{V}_{it} = 0$ when $S_{it} = \emptyset$.
We want to show that the system suffers linear regret when
the platform uses any scoring mechanism
for which scores converge to the empirical average of the observed values.
This means that the system never converges to an optimal policy;
rather a constant fraction of users are misled into perpetuity.
To make this rigorous, we define the notion of \emph{mean-converging} scoring process.
\begin{definition}
    A scoring process that outputs scores $s_{it}$
    for item $i$ at time $t$ is \emph{mean-converging} if
    \begin{enumerate}
        \item $s_{it}$ is a function of $\{V_\tau : \tau \in S_{it}\}$ and $t$.
        \item $s_{it} \to \bar{V}_{it}$ almost surely if $\liminf_t \frac{|S_{it}|}{t} > 0$ almost surely.
    \end{enumerate}
\end{definition}

In words, the score only depends on the observed outcomes for this particular
item, and if we observe a linear number of selections of arm $i$, then the
score converges to the mean outcome.  Trivially, this includes using the
average itself as score, $s_{it} = \bar{V}_{it}$, but this definition also includes well known
methods that carefully balance exploitation with exploration, such as versions
of UCB and Thompson Sampling.

% \begin{lemma}
%     An upper confidence bound strategy with an upper bound of the form
%     $ \alpha\sqrt{\frac{\log(T)}{|S_{it}|}} $ is mean-converging.
% \end{lemma}
% This is immediate from the definition of mean-converging.
% The same is true for Thompson sampling as long as the prior is independent and well specified.
% \begin{lemma}
%     If the noise distribution $E_i$ has a normal distribution,
%     then Thompson sampling with an independent normal prior
%     for each $Q_i$ is mean-converging.
% \end{lemma}
% This follows because the prior washes out and
% therefore the posterior will converge to the empirical point estimate.

From the previous section, we know that, ideally,
the scores supplied to the user converge to the quality of the item, $s_{it} \to Q_i$, as more
users select item $i$.
We say that the scores are \emph{biased} if this is not the case:
\[
    s_{it} \not \to Q_i \quad\text{ as }\quad |S_{it}| \to \infty.
\]
The next proposition shows that mean-converging scoring processes lead to linear regret,
because these scores are generally biased.
We only show this result for when preferences are drawn from Bernoulli distributions as
this simplifies the analysis significantly.
In Section~\ref{sec:simulations} simulations show that linear regret is observed under
a variety of preference distributions.
Under the Bernoulli model, it is needed that the gap between qualities
is `small', though we show that this condition is rather weak.
\begin{proposition}
    When $\theta_{it} \sim $ Bernoulli$(p)$ for all $i, t$,
    if
    \[
        \Delta = Q_1 - Q_2 < \frac{(1-p)^K}{(1-p)^K + p}
    \]
    and $s_{it}$ is mean-converging,
    then
    \[
        \limsup_{t\to\infty} \frac{R_t}{t} \ge c
    \]
    for some $c > 0$.
\label{thm:bias}
\end{proposition}
The proof of this proposition can be found in the supplemental material.
The intuition behind the result is that the best ranked item is selected by users
that do not necessarily like it that much, while other items are only selected by
users who really love it.
Therefore, the ratings of the best ranked item suffers relative to others.

We note that for $K=100$ and $p=1/50$, the condition requires $\Delta < 0.86$.
More generally, in the relevant regime where $p < \frac{\log(K)}{2K}$,
the condition is satisfied if $\Delta < 0.7$ for all $K$.
We also note that the linear regret we obtain is not caused by the usual
exploration/exploitation trade-off, but rather the estimators being biased.

There is no bias result for general preference distributions
as it is possible to cherry pick distributions in such
a way that biases cancel each other out exactly.
Furthermore, the magnitude of the problem depends crucially on the variance in
user preference relative to the differences in qualities.\footnote{
    In the limiting scenario of no variance in preferences,
    we already know that there is no bias either.
}
However, in Section~\ref{sec:simulations} we provide simulations with a variety
of preference distributions that suggest that bias is not an artifact of our assumptions.

\subsection{Unbiased estimates}

Naturally, a first attempt to improve the linear regret
is aimed at obtaining unbiased versions of the naive averaging.
We now sketch a few such approaches.

\subsubsection{Randomization}
Researchers can avoid selection bias in experimental studies by randomizing treatments,
and we can employ the same approach here.
Instead of the user choosing an action,
the platform assigns matches between users and items.
Note that pure randomization is not needed, user ratings are unbiased
as long as the selection of items is independent from private preferences.

Just like randomized control studies, this approach is often infeasible or
prohibitively expensive, e.g. a platform cannot force the user to select
a certain hotel or restaurant.
However, small scale user experimentation can potentially inform the platform
on the magnitude and effects of idiosyncratic preferences and the bias it induces.

\subsubsection{Algorithmic approach}
Another option is to consider algorithmic approaches to obtain consistent estimators.
In the case of Bernoulli preferences,
it is possible to obtain unbiased estimates from data.
However, this approach does not generalize to other preference distributions,
let alone to the case where we do not know the underlying preference distributions.

% Thus, when presented with real data, we would be hesitant to impose such rigid
% structure on the preferences of users.
% While the Bernoulli model is helpful in yielding a tractable model that gives insight
% into the main dynamics, it is not realistic to think it accurately reflects the preferences in practice.
% At a minimum, it makes sense to use a continuous preference distribution.
% However, algorithmically unbiasing the data appears impracticable and intractable in this setting;
% impracticable because it still requires strong assumptions on the preference distributions,
% and intractable because it requires integrating over the preference distributions for each item.

\subsubsection{Changing the feedback model}
Given the difficulty of debiasing feedback algorithmically, we briefly discuss
a third alternative.  The traditional type of question `How would you rate this
item?' asks for an absolute measure of satisfaction, which corresponds to
directly probing for $V_t$ in our model.  If we ask how the chosen item
compares to the expectation, we ask for a relative measure of feedback,
approximating $V_t - (s_{a_t t} + \theta_{a_t t})$.  An example of such prompt
could be `How does this item compare to your expectation?' This way, we can
uncover an unbiased estimate of $Q_{a_t}$.  Importantly, it does not require
any distributional assumptions on the form of the preferences.  Whether such
relative feedback works in practice would require a thorough empirical study,
which is beyond the scope of this work.  We also note that this approach does
not work when platforms collect implicit feedback.

% While relative feedback seems appealing, there are two caveats.
% Users might not be capable of giving such feedback,
% or maybe users already give such feedback implicitly, or that
% `a priori expectation' is something users are not capable of accurately reporting.
% The second caveat comes from the fact that many recommendation systems now rely
% more on implicit feedback, e.g.,~did someone finish watching the movie.
% Such measures are necessarily absolute and our suggestion is vacuous in such scenarios.
% A thorough study of this approach is beyond the scope of this work.

%% file: tex/exploration.tex
\section{Efficiency}
\label{sec:exploration}

From the previous section we know that naive scoring mechanisms
are inconsistent and lead to linear regret.
We now focus on the efficiency of consistent estimators,
and we assume we have access to unbiased feedback from now on.
But, this is not necessarily sufficient to guarantee good performance.
The multi-armed bandit literature suggests that algorithms with small regret
require a careful balance between exploration and exploitation.

In particular, that means that the system needs to obtain data on
every item in order to provide useful scores to the users.
However, myopic agents have no interest in assisting the platform with exploration.
% Therefore, there has recently been an increased interest in understanding this trade-off
% in the presence of myopic agents, who are naturally interested in
% doing well for themselves, rather than helping the platform learn
% \cite{Frazier2014SIGECOM, Slivkins15, MansourEC16, papanastasiou2014crowdsourcing}.
% The previously mentioned works all address the question of how to incentivize
% users to not act myopically in different ways.
In this section we address the problem of exploration in the proposed model.
As opposed to the research mentioned in the introduction,
we deal with agents with heterogeneous preferences.
It seems natural that these heterogeneous agents help the system explore,
but it is not obvious to what extent this helps.
We show that because of this diversity in preferences,
the free exploration leads to optimal performance (up to constants);
we recover the standard logarithmic regret bound from the bandit literature.
This means that there is little need for a platform to implement a complicated
exploration strategy, and incentives naturally align much better than in
the settings of previous work.

\subsection{Formal result}
We assume the qualities $Q_i$ are bounded, and without loss of generality we can
assume they are bounded in $[0, 1]$, and
we assume access to unbiased feedback from the user.
That is, the feedback at time $t$ for chosen item $i_t$ is
$\tilde{V_t} = Q_{i_t} + \epsilon_t$.
However, we no longer require that private preference distributions are Bernoulli.
Let $\bar V_{it}$ denote the average of the values observed of item $i$ by time
$t$
\[
  \bar V_{it} = \frac{1}{|S_{it}|} \sum_{\tau \in S_{it}} \tilde{V}_\tau
\]
where $S_{it} = \{ \tau < t : a_\tau = i\}$.
We now consider the scoring algorithm that clips the value onto $[0, 1]$:
$
  % s_{it} = \begin{cases}
  %   0 & \bar V_{it} < 0\\
  %   1 & \bar V_{it} > 1\\
  %   \bar V_{it} & \text{otherwise}
  % \end{cases}
  s_{it} = \max(0, \min(1, V_{it}))
$.

Then, if the private preferences have sufficiently large variance,
made precise in the theorem statement, then exploration is guaranteed
and the platform suffers logarithmic regret.
Let $\Delta_{\min}$ be the smallest gap in qualities
$
  \Delta_{\min} = \min_{i, j} |Q_i - Q_j|.
$

% We consider a platform that uses empirical averages of these unbiased scores,
% \[
%     \tilde{s}_{it} = \frac{1}{|S_{it}|} \sum_\tau \tilde{V}_\tau
% \]
% However, this is not enough to ensure a tight regret bound,
% as a single dramatically low rating for some item
% can cause all future agents to ignore that item forever.
% Therefore, we impose the condition that
% \[
%     1 + \min_i \tilde{s}_{it} > \max_i \tilde{s}_{it}
% \]
% by increasing the lowest scores to satisfy this criterion if needed.
% In practice, this ensures that for any item and any set of valid scores
% there is a realization of preferences for a user such that the item is selected.
% It is impossible to prove a general regret bounds for the empirical scores
% with discrete signals and unbounded support for error terms;
% with some small but positive probability, an item gets such a terrible rating
% that the signals cannot make a difference.
% From a practical standpoint, it would also not make sense to
% offer items for which the platform knows no one is interested.
% Note that this does not lead to incentive issues
% because we assume all qualities lie in the unit interval.

The following result shows that empirical averaging
is enough to get an order optimal (pseudo)-regret bound with respect to
the total number of agents $T$.
\begin{proposition}
  \label{thm:exploration}
  If for all $i$, $F_i$ are such that $\P(\theta_{it} > 1) > \gamma'$
  and $\P(\theta_{it} \le 0) \ge \gamma$
  Then
  \begin{multline}
    \E[\regret(T)]
    \le \left( \frac{16\sigma^2}{\Delta_{\min}} + \Delta_{\min} \right) K\\
    + \frac{32 \alpha \sigma^2 K(\log(T) - \log(\Delta_{\min})
    + \log(2))}{\Delta_{\min}^2 C}
  \end{multline}
  where $C = \gamma' \gamma^{K-1}$.
\end{proposition}
The proof can be found in the supplement.
The main idea is that we can show every arm is chosen sufficiently often
initially, and then concentration inequalities ensure good performance
after an initial learning phase.% ~\ref{sec:proof_exploration}.

\begin{corollary}
  Suppose $F_i$ is a symmetric distribution around $0$ such that
  $\P(\theta_{it} > 1) > \gamma$, then the theorem applies with
  $C \ge \gamma 2^{1-K}$.
\end{corollary}

Also note that the theorem applies for Bernoulli preferences, where $C = p (1-p)^{K-1}$.
The small value of $C$ leads to a large leading constant.
Simulations suggest that performance is much better in practice.

% We remark that this is a problem dependent bound,
% and it is not optimized in terms of $K$ and the constants.
% In simulations (discussed in Section~\ref{sec:simulations}) we show
% that, depeding on the variance of preferences, the performance of this simple
% learning algorithm is excellent.\footnote{
%   Based on standard multi-armed bandit model, we know that
%   regret must be linear if there is no variance in preferences.
% }

\subsection{Remarks}
The main take-away from this result is not so much the specific bound,
but rather the practical insight that it, and its proof, yield.
The intuition is that initially estimates of quality are poor.
Therefore, it takes some time and luck for users with idiosyncratic preferences to
try these items.
As estimates improve, however, most agents are drawn to their optimal choice.
Since these choices differ across agents, the platform gets to learn efficiently
about all items without incurring a regret penalty.

The practical consequence of this observation is that to improve the performance,
the designer of a recommendation system should focus on simple ways to make new items,
or more generally items with few observations, more likely to be chosen.
We can achieve this by highlighting new arrivals.
A good example is Netflix, which clearly displays a `Recently Added' selection.

%% file: tex/simulations.tex
\section{Simulations}
\label{sec:simulations}

In this section we empirically demonstrate that the theoretical results derived
in the previous sections hold much more broadly.
First, we focus on verifying the results from our simplified model.
Thereafter, we consider more advanced personalization models using
feature-based and low rank methods, where we
investigate the dynamics with private preferences.\footnote{
    The code to replicate the simulations is publicly available
    at \url{https://github.com/schmit/human_interaction}.
}

\subsection{Simulations of regret}

Before considering more advanced methods, we simulate our model using
different preference distributions and plot the cumulative regret over time.
We run 50 simulations with 5000 time steps and $K=50$ items across four preference distributions
with randomly drawn parameters.
We then compare biased and unbiased algorithms based on empirical averages.
Figure~\ref{fig:plot_bias_regret} shows the cumulative regret paths for each of these simulations.
The qualities were drawn from the uniform distribution over $[0,1]$.
For the preference distributions $F_i$ for item $i$, we used
\begin{itemize} \itemsep0em
    \item Bernoulli distribution with $p_i \sim U[0, \frac{2\log(K)}{3K}]$.
    \item Normal distribution with $\mu_i = 0$ and $\sigma_i \sim U[0, 1]$.
    \item Exponential distribution with scale $\lambda_i^{-1} \sim U[0, 1]$.
    \item Pareto distribution with shape $\alpha_i \sim U[2, 4]$.
\end{itemize}
These are chosen such that the variance in preference and qualities is
roughly similar.
A clear pattern emerges;
In all cases, the (biased) empirical averages lead to linear regret,
not just for the Bernoulli model covered by Proposition~\ref{thm:bias}.
Second, we note the unbiased scores lead to much better results regardless
of the preference distribution, in line with Proposition~\ref{thm:exploration}.

\begin{figure}[h]
    \centering
    \includegraphics[width=0.8\textwidth]{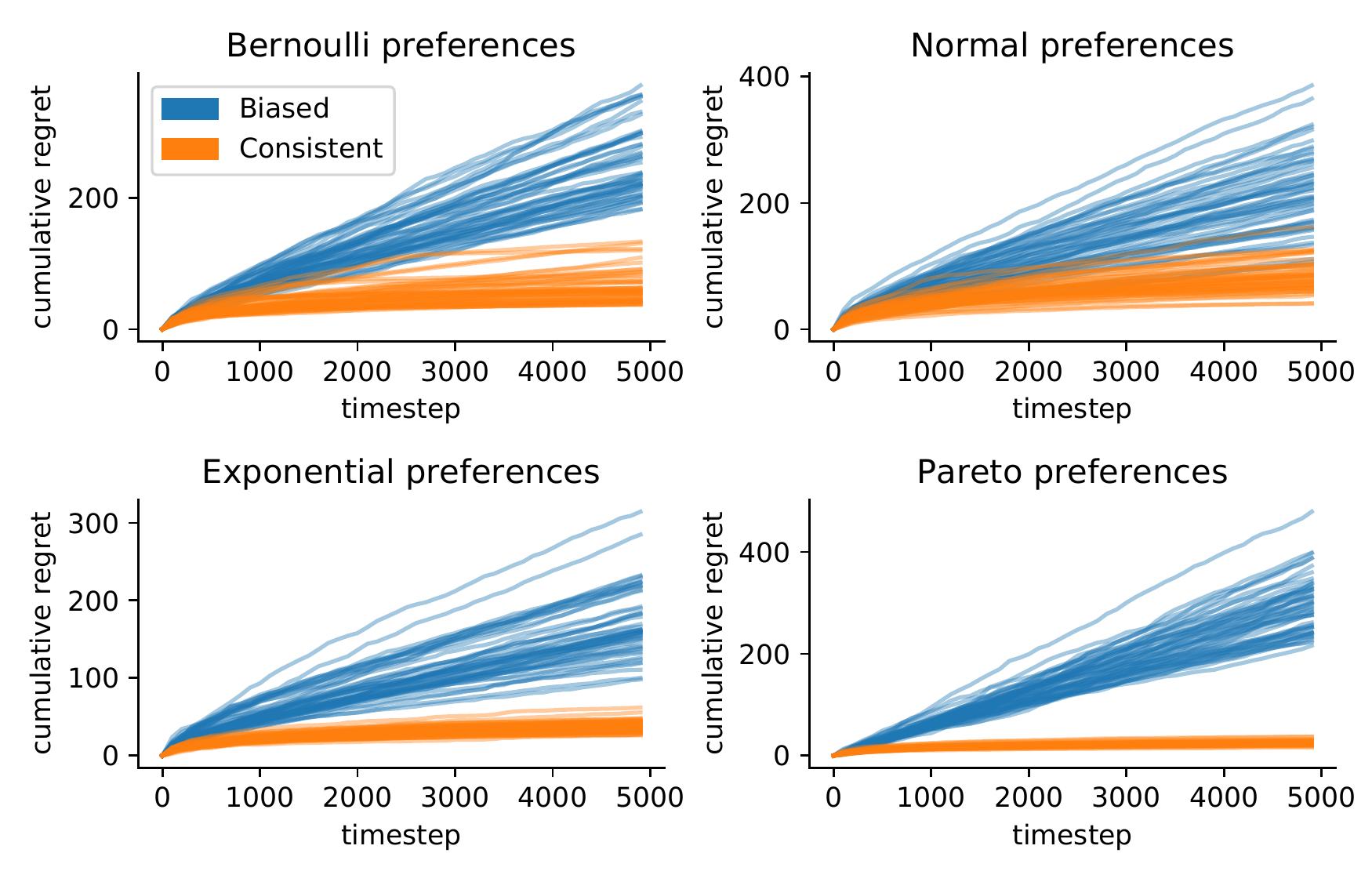}
    \caption{
        These plots shows the cumulative regret plotted against time for both the
        naive empirical averages in blue, and the unbiased averages in
        orange.
        % Different distributions for users' preferences are used
        % to show that bias and exploration results hold broadly.
    }
    \label{fig:plot_bias_regret}
\end{figure}

\subsection{Personalization methods}

We now focus on methods that provide more personalized recommendations.
Because estimating such models is more complicated and computationally intensive,
we simplify the dynamics of our simulations to a two-staged approach.
We then use this approach to experiment with a feature-based
and a low-rank approximation approach to personalization based on
synthetically generated data.

\subsubsection{Two-staged simulations}

In our original setup, the platform updates its scoring rule after every observation.
This is impractical when dealing with more sophisticated models.
Instead we first collect a set of observations using a fixed scoring algorithm (a \emph{training set}),
and fit the scoring algorithm once to this training data.
We then use this trained algorithm to generate a new set of observations (a \emph{test set}),
again without updating the algorithm in between observations.
The test set is used to measure the performance of the fitted algorithm.
Instead of using regret, we measure performance by directly computing the average rating
on the second dataset generated by our trained algorithm.
% \begin{enumerate} \itemsep0em
%     \item We generate a rating dataset (training set) based on a fixed score rule.
%     \item We fit a model once to the entire dataset to obtain a new scoring rule,
%         which we then use to generate a new dataset (test set) with ratings.
%     \item We generate a new dataset (test set) using the new scoring rule.
%     \item We compute the average rating on the test set as measure of performance.
% \end{enumerate}
Note that it is possible to iterate generating data and fitting models multiple times
before generating a test set.

\subsubsection{Ridge regression}
\label{sec:lr}

In this section we discuss a feature-based model of personalization where
the rating is assumed to be a linear function of observed covariates.

\textbf{The model}
In the feature-based setting, each item has an unknown parameter vector $w_i$
and each user-item pair has an observed feature vector $x_{it}$.
The value of item $i$ for user $t$ then becomes
\[
    V_{it} = Q_i + x_{it}^T w_i + \theta_{it} + \epsilon_{it}.
\]
Furthermore, we also parametrize $\theta_{it}$ in terms of $x_{it}$,
such that
\[
    \theta_{it} \sim \N(x_{it}^T \tilde{w}, \sigma_{\theta})
\]
where $\tilde{w}$ is another unknown parameter vector.
After generating the training set using a fixed scoring rule,
we use ridge regression to regress the reported
ratings for each item, which leads to estimates $\hat Q_i$ and $\hat w_i$ for each item.
These are then used as scoring rule: $s_{it} = \hat Q_i + x_{it}^T \hat w_i$.

\textbf{Simulation details}
There are $n = 100$ items, and $w_i \in \reals^p$ where $p = 20$.
We generate $Q_i \sim \N(0, 1)$ and $w_{ij}, \tilde{w}_{ij} \sim \N(0, 1/\sqrt{p})$ independently.
The elements of the feature vectors are generated independently following
$x_{ijt} \sim \N(0, 1)$.
The error term is drawn according to $\epsilon_{it} \sim \N(0, 1)$,
and we set $\sigma_\theta = 0.1$.
We generate $10np = 20000$ observations.

The training sets are generated using four different scoring rules:
\begin{enumerate}\itemsep0em
    \item Using the \emph{oracle} scoring rule: $s_{it} = Q_i + x_{it}^T w_i$,
        which leads to perfect recommendations.
    \item Using the oracle scoring rule
        and unbiased ratings $V_{it} - \theta_{it}$.
    \item Using randomly selected items, hence the user has no choice.
    \item We iterate steps one and two twice, where we first use randomly
        selected items, then fit a Ridge regression to estimate the
        parameters, and use these to generate the test set: $s_{it} = \hat Q_i
        + \hat x_{it}^T w_i$.
\end{enumerate}
This last training set allows us to better understand how the system evolves over time.

\textbf{Results}
The average values of the selected items in the test set are
plotted in the left plot of Figure~\ref{fig:lr_mf}.
The two dotted lines provide useful benchmarks:
the top line shows the performance of the oracle scoring rule,
which upper bounds the performance.
The bottom line shows the performance in the absence of a scoring rule,
that is $s_{it} = 0$ for all $i, t$.
Finally, note that random selections lead to an average rating of $0$.

We note that the best performing recommendations are given by the model
trained on random data (green); these are close to the performance of the oracle.
Unbiased ratings based on the oracle (orange) perform a bit worse
due to a feedback loop.
The model is only trained on `good' selections and this leads to a degradation
of performance.
We also see that the iterated model (red) that was initially trained on random selections
performs worse than the randomly generated data,
suggesting that the quality deteriorates over time.
Finally, the model trained on oracle data (blue) performs much worse than all the models,
and does not perform much better than the `no-score algorithm' that does not provide recommendations.

\begin{figure}
    \centering
    \includegraphics[width=0.8\textwidth]{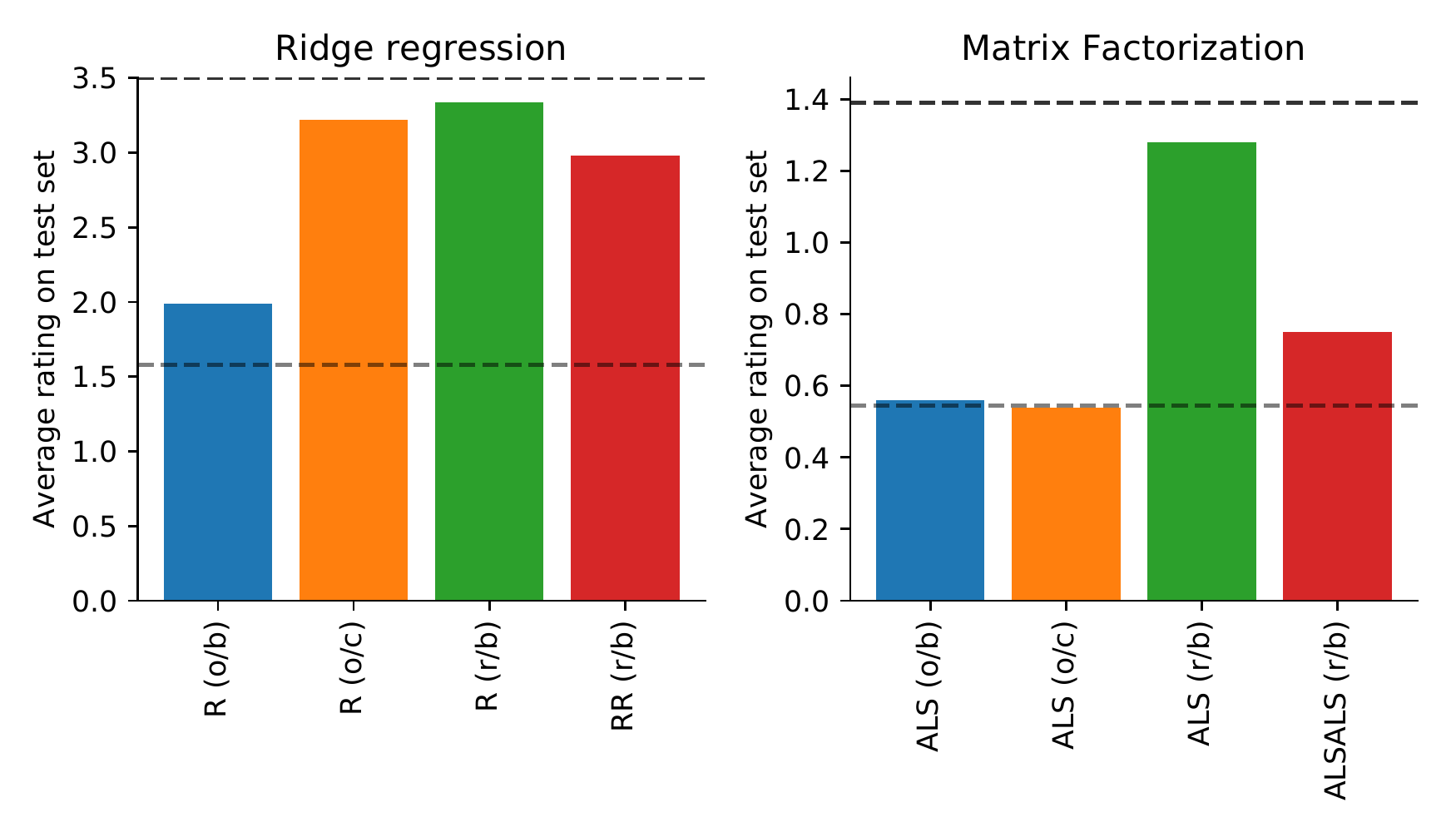}
    \caption{
We compare performances of recommendation systems based on different training data
in a feature-based model on the left and matrix factorization on the right. Green corresponds to random selections in the training set,
orange to oracle selections with unbiased feedback, red to the iterated model initially
based on random selections, and blue is based on biased oracle ratings.
}
    \label{fig:lr_mf}
\end{figure}

\subsubsection{Matrix factorization}
\label{sec:mf}

In this section we investigate the dynamics of private preferences that are low rank.
We use the same two-stage approach as before, where we first use a fixed
scoring rule to generate a training set, fit our model, and use the fitted scoring rule
to generate a test set to measure performance.

\textbf{The model}
The low-rank models assumes that the value for item $i$ by user $j$ follows the model
\[
    v_{ij} = a_i + b_j + u_i^Tv_j + x_i^Ty_j + \epsilon_{it}
\]
where $u_i$ and $v_j$ are (hidden) $q$-dimensional vectors modeling interaction between user and item,
and $a_i$ and $b_j$ are terms modeling overall differences between users and items.
Similarly, $x_i$ and $y_j$ are $q$-dimensional vectors that combine the private preference $\theta_{it} = x_j^Ty_j$.
Note that, unlike in previous settings, here we observe the same user multiple times.
The recommendation system provides score $s_{ij}$ for user $i$ and item $j$ and the user selects item
\[
    \arg\max_j s_{ij} + x_i^Ty_j
\]
and reports her value $V_{ij}$.
We use alternating least squares \citep{Koren2009MatrixFT} to estimate $a$, $b$, $u$ and $v$.\footnote{
    We ensure that users do not rate the same item in both the training and test set.}

\textbf{Simulation details}
As in the feature-based simulation, we generate four training sets,
one based on an oracle, one based on an oracle with unbiased ratings,
one based on random selections, and finally an iterated version of the random selections process,
where we fit a model to the random selections data and use that model to generate training data.

We simulate 2000 users and 500 items, with rank $q=4$.
Entries of $u, v, x$ and $y$ are independent Gaussians with variance $1/q$.
To reduce variance, the error term has a small variance, $\epsilon_{ij} \sim N(0, 0.01)$ and
for both training and test sets each user rates $40$ items.
We run alternating least squares with rank $2q$ and varying regularization.

% \begin{figure}[h]
%     \centering
%     \includegraphics[width=0.5\textwidth]{img/mf_als.pdf}
%     \caption{
% Comparison of ratings based on Alternating Least Squares.
% Dashed lines provide benchmarks of recommending perfect scores (p), zero scores (n-s) and random selections (r).
% Blue lines corresponds to perfect scores with bias.
% Orange line corresponds to perfect scores with no bias.
% Green and red lines correspond to random selections with and without bias, respectively.
% }
%     \label{fig:mf}
% \end{figure}

\textbf{Results}
The two dotted lines denote the same benchmarks as before.
Again, we notice that recommendations trained on random data perform best,
but this time the difference is much more pronounced.
The recommendations based on perfect recommendations (blue and orange) perform a lot worse.
In fact, they do barely better than not recommending items at all and having
users base their choice solely on their own preference signals.
Part of the degradation in performance seems to be caused by a feedback loop;
the observations are not randomly sampled.
We also notice a much stronger degradation in performance of the iterated model (red).
This suggests that the dynamic nature of recommendation systems affect matrix factorization
methods more severely than the simpler linear model from the previous section.

%% file: tex/conclusion.tex
\section{Discussion}
\label{sec:conclusion}

In this work, we introduce a model for analyzing feedback in recommendation systems.
We propose a simple model that explicitly looks at heterogeneous preferences
among users of a recommendation system,
and takes the dynamics of learning into account.
We then consider the consistency and efficiency of natural estimators in this model.
Recent work has focused on exploration, or efficiency, with selfish agents.
On the one hand, preferences lead to inconsistent estimators if this aspect is not taken into account.
On the other hand, we also show that there is an upside to heterogeneous preferences;
they automatically lead to efficiency.
Using simulations, we demonstrate that these phenomena persist when we use
more sophisticated recommendation methods, such as matrix factorization.

\subsection{Future work}
\label{sec:future}

There are several directions of further research.
Our simplified model does not capture all aspects of recommendation systems.
The most interesting aspect is that, in practice,
users only observe a limited set of recommendations,
rather than the entire inventory.
This can lead to an inefficiency in the rate of exploration,
and requires further study.
% This is something we have not modeled and has clear implications on the rate of exploration.

Our model and simulations show that consistency of models is an issue
that is difficult to resolve.
We believe that progress can be made.
Theoretically, one possible avenue is to also model the selection process
directly and combine it with the model for outcomes.
Empirically, by large scale studies that test the effects of human
interaction on estimators.

% Beyond exploration, things get more interesting when there are features
% that are correlated with both the feedback ($V_t$) and the user selection ($a_t$).
% To do well while only supplying a limited set of recommendations,
% the server has to combine an outcome model (based on the rating) with a selection model.
% It is unclear how to do this optimally, as it requires balancing showing items
% that users are likely to select with items that they are likely to rate highly.
% For example, if a certain feature is a strong predictor for selection,
% it is likely a weak predictor for outcome conditioned on selection;
% if users base their selection heavily on a particular feature,
% then there is little variance left to exploit for
% the feedback part of the model.\footnote{
% 	This is best illustrated with an example. Suppose users base their movie selection on genre, such that
% 	a user that loves comedies only selects comedies, then there is no feedback for thrillers.
% 	Therefore, the server has trouble picking up on the correlation between genre and feedback,
% 	even though this effect could be strong.}

\subsection{The bigger picture}
We believe that this work has raised fundamental and important issues
relating the interaction between
machine learning systems and the users interacting with them.
Algorithms not only consume data, but in their interaction with users
also create data, a much more opaque process but equally vital in designing
systems that achieve the goals we set out to achieve.
There is still a lot of room for improvement by gaining a better understanding of these dynamics.

%% file: tex/supplement.tex
\appendix

\section{Appendix}

\input{tex/suppl/proof_bias}

\input{tex/suppl/proof_exploration}

%% file: tex/suppl/proof_bias.tex
% \subsection{Proof bias}
% \label{sec:proof_bias}

\begin{proof}[Proof of Proposition~\ref{thm:bias}]
    We prove the result by showing that the best item
    cannot always be ranked at the top, because that would
    depress its score $s_{it}$ sufficiently much that it cannot be at the top.

    Fix a sample path $\omega \in \Omega$.
    Note that by assumption, each arm is optimal for a constant fraction of agents.
    Define
    \[
        x_{it} = \frac{ | \{ \tau : a_\tau = i \} | }{ t }.
    \]
    Then, if $\liminf_t x_{it} < x_i^*$ for some sufficiently small $x_i^* > 0$,
    we incur linear regret almost surely.
    Instead, assume that each arm is sampled a constant fraction,
    $\liminf_t x_{it} > \delta_i$ for some $\delta_i$ for each arm $i$.
    We note that the expected reward for the item ranked highest is
    \[
        Q_i + \frac{ p } { p + (1-p)^K } = Q + \rho,
    \]
    where we define $\rho = \frac{ p }{ p + (1-p)^K }$:
    With probability $p$ this item is chosen because of a positive signal,
    and with probability $(1-p)^K$ it is chosen because none of the items
    have a positive signal.
    For the other items, the expected reward is $Q_i + 1$.

    To understand limiting behavior of the item scores, it is
    thus important to understand how often an item is ranked first by
    the platform.
    Define $c_{t}$ as the fraction (up to time $t$)
    that the first (best) item is \emph{not} ranked at the top:
    \[
        c_t = \frac{ | \{ \tau  < t: \exists j > 1 :  s_{1\tau} < s_{j\tau} \} | }{ t }.
    \]
    We note that if $\limsup_t c_t > c^*$ for some $c^*>0$, then
    the regret is linear.

    Informally, we proceed by bounding $\P(\text{item is ranked first} \mid \text{item is selected})$,
    and use that to understand the evolution of the averages of ratings the platform
    observes.
    To bound the above probability, we note that there are two extremes
    when the item is not ranked first; it is ranked second, or ranked last.
    If it is always ranked second when the item is not ranked first,
    it is less likely the item was ranked first given selection
    than when it is either ranked first or last.
    If, overall, the item is ranked first with fraction $y$,
    then we obtain
    \[
        \lambda(y) \le \P(\text{item ranked first} \mid \text{item selected}) \le \lambda'(y)
    \]
    where
    \[
        \lambda(y) = \frac{y (p + (1-p)^K)}{y(p + (1-p)^K) + (1-y) p (1-p)},
    \]
    and
    \[
        \lambda'(y) = \frac{y (p + (1-p)^K)}{y(p + (1-p)^K) + (1-y) p (1-p)^{K-1}}
    \]
    correspond to the two extreme cases.
    Note that $\lambda$ and $\lambda'$ are both decreasing.\footnote{
        Both have the form $\frac{(1-x)a}{(1-x)a + xb}$ for $a, b \in (0,1)^2$,
        which has a negative derivative for $x \in (0, 1)$}

    Now suppose $\limsup c_t = c$.
    By the stong law of large numbers, the empirical average converges to its mean and thus
    \[
        \limsup_t s_{1t} \le Q_1 + \lambda(1-c) \rho + (1-\lambda(1-c)),
    \]
    where the second term corresponds to the expected reward from being ranked first
    and the last term corresponds to the contribution from when the action is not ranked first.
    Similarly
    \[
        \liminf_t s_{2t} \ge Q_2 + \lambda'(c) \rho + (1-\lambda'(c)),
    \]
    almost surely by the mean-converging condition.

    We note for $c = 0$, this leads to
    \[
        \limsup_t s_{1t} \le Q_1 + \rho
        \text{  and  }
        \liminf_t s_{2t} \ge Q_2 + 1
    \]
    This is a contradiction if $\Delta < \frac{(1-p)^K}{p + (1-p)^K}$,
    as this would imply the score of the second arm is higher in the limit than
    that of the first arm, while the first item is always ranked before the second item ($c=0$):
    \begin{align}
        \limsup_t s_{1t}
        % &\le Q_1 + \rho \\
        &= Q_2 + \Delta + \rho \\
        &< Q_2 + \frac{(1-p)^K}{p + (1-p)^K} + \frac{p}{p + (1-p)^K} \\
        % &= Q_2 + 1 \\
        &\le \liminf_t s_{2t}.
    \end{align}
    Furthermore, since $\lambda$ and $\lambda'$ are continuous and monotone,
    there must exist some $c^* \in (0, 1)$ such that
    \begin{multline}
        Q_1 + \lambda(1-c^*)\rho + 1-\lambda(1-c^*) \le \\
        Q_2 + \lambda'(c^*) + 1-\lambda'(c^*)
    \end{multline}
    almost surely.
    Thus, if the first item is the top ranked item fracion $1-c^*$ of the time,
    then its score is almost surely lower than the second item, which is
    a contradiction.
    This implies that $\limsup_t c_t > c^*$ almost surely,
    which proves that the regret is linear.
\end{proof}

%% file: tex/suppl/proof_exploration.tex
\begin{proof}[Proof Proposition~\ref{thm:exploration}]
  To bound the regret, we look at individual arms and note that
  if at time $t$ all scores $s_{it}$ are reasonably accurate,
  i.e. $| s_{it} - Q_i | < \lambda$ for all $i$,
  at such time the regret is at most $2\lambda$.
  Furthermore, if $\lambda < \frac{\Delta_{\min}}{2}$, then
  the regret is $0$ as each agent is compelled to pick the best item for them.
  Finally, it is important to note that the regret at each period is at most $2$.

  We proceed as follows; we use concentration to bound the estimation error
  when we have observed enough sample values.
  Furthermore, we show that due to natural exploration,
  we have a high probability guarantee of observing samples for each item.
  When combined, they lead to a logarithmic regret bound.

  To use a concentration bound on the estimation error,
  we define event
  \[
    A_m(i, \lambda) =
    \left\{
      \exists s \in \{m, \ldots, T\} :
      \frac{1}{s} \left| \sum_{j=1}^s \epsilon_{ij} \right| > \lambda
    \right\}.
  \]
  That is, $A_m(i, \lambda)$ is the bad event that after $m$ pulls,
  there is some time $t$ that the score $s_{it}$ is off by more than $\lambda$.

  Furthermore, we define events
  \[
    B_{m}(i, M) =
    \left\{
      |S| < m : \tau \in S \iff a_\tau = i \text{ and } \tau < M
    \right\}
  \]
  that indicate whether within $M$ time steps,
  at least $m$ users reported values for item $i$.

  Using these two events, we can bound the expected regret by
  \begin{multline}
    \E [\regret(T)] \le
    \sum_{i=1}^K 2 ( \P(A_m(i, \lambda)) + \P(B_m(i, M)) ) T \\
    + 2M + \lambda T \ind_{\lambda > \frac{\Delta_{\min}}{2}}
    \label{eqn:initregretbound}
  \end{multline}

  \textbf{Bounding $A_m$}
  Using the standard $\sigma$-sub-Gaussian concentration bound (see,
  for example, \citet[Chapter 2]{Wain}),
  we have
  \begin{align}
      \P(A_{m}(i, \lambda))
      &\le
          \P \left( \exists s \in \{m,\ldots,t\} :
              \frac{1}{s} \left| \sum_{i=1}^s \epsilon_i \right| > \lambda
          \right) \\
      &\le
          \sum_{s=m}^t \P \left(
              \frac{1}{s} \left| \sum_{i=1}^s \epsilon_i \right| > \lambda
          \right) \\
      &\le
          2 \sum_{s=m}^t \exp \left( - \frac{s\lambda^2}{2\sigma^2} \right) \\
      &\le
          2 \int_m^{t+1} \exp \left( - \frac{s\lambda^2}{2\sigma^2} \right) ds \\
      &\le
          \frac{4\sigma^2}{\lambda^2} \exp \left( -\frac{m\lambda^2}{2\sigma^2} \right)
  \end{align}
  Now set
  \[
      m = \frac{2 \sigma^2 (\log(T) - \log(\lambda))}{\lambda^2},
  \]
  and obtain
  \[
      \P(A_m(i, \lambda))
      \le \frac{4\sigma^2}{\lambda^2} \exp \left( -\frac{m\lambda^2}{2\sigma^2} \right)
      = \frac{4 \sigma^2}{\lambda T}
  \]

  \textbf{Bounding $B_m$}
  From the above, we know that the estimation error concentrates
  well after observing $m$ selections.
  Now we show that with high probability,
  it does not take too long to wait for $m$ selections.

  First note that the probability of selection of
  any item at any time $t$ is at least $2^{1-K} \gamma$.
  This follows from the conditions imposed on $F_i$.
  For $M > m$, we note that the probability
  that we have not observed $m$ selections is lower bounded by a Binomial random
  variable $Z \sim B(M, 2^{1-K}\gamma)$ since preferences are independent between agents.
  Consider
  \[
    M = \frac{2\alpha m}{2^{1-K}\gamma}
    = \frac{4\alpha \sigma^2 (\log(T) - \log(\lambda))}{\lambda^2 2^{1-K}\gamma}
  \]
  where $ \alpha = \max \left(1, 2\lambda^2 / \sigma^2 \right)$ .

  First we note that in this case, $ \E(Z) = 2\alpha m \ge 2m $
  and thus
  \begin{align}
      \P(B_m)
      % &\le
      %     \P(Z \le m) \\
      &\le
          \P \left( Z \le \frac{1}{2} \E(Z) \right) \\
      &\le
          \exp \left(
              - \frac{ \E (Z) }{ 8 }
          \right)\\
      &\le
          \exp \left(
              - \frac{ \alpha \sigma^2 (\log(T) - \log(\lambda) ) }{ 2 \lambda^2 }
          \right)\\
      % &\le
      %     \exp \left(
      %         - (\log(T) - \log(\lambda) )
      %     \right)\\
      &\le \frac{\lambda}{T}
  \end{align}
  where third inequality is a standard Chernoff bound and
  the second to last step follows from the condition on $\alpha$.

  Plugging these bounds on $A_m(i, \lambda)$ and $B_m(i, M)$
  in to our bound for regret (\ref{eqn:initregretbound}),
  we obtain
  \[
  \begin{split}
      \E[\regret(T)]
      &\le 2\left( \frac{4\sigma^2}{\lambda} + \lambda \right)K\\
      &+ \frac{8\alpha \sigma^2 K(\log(T) - \log(\lambda))}{\lambda^2 2^{1-K}\gamma}
        + \lambda KT \ind_{\lambda > \frac{\Delta_{\min}}{2}}
  \end{split}
  \]
  and thus if we set $\lambda = \frac{\Delta_{\min}}{2}$, we find
  \[
  \begin{split}
      \E[\regret(T)]
      &\le \left( \frac{16\sigma^2}{\Delta_{\min}} + \Delta_{\min} \right) K\\
      &+ \frac{32 \alpha \sigma^2 K(\log(T) - \log(\Delta_{\min})
      + \log(2))}{\Delta_{\min}^2 2^{1-K}\gamma}
  \end{split}
  \]
  as desired.
\end{proof}